\documentclass[letterpaper,10pt,conference]{ieeeconf}

\usepackage{cite}
\usepackage{graphicx}
\usepackage{subcaption}
\usepackage[cmex10]{amsmath}
\usepackage{algpseudocode}
\usepackage{algorithmicx}
\usepackage{algorithm}
\usepackage{xcolor}
\usepackage{array}
\usepackage{url}
\usepackage{bbm}

\usepackage{todonotes}

\usepackage{epsfig}
\usepackage{amsfonts,amssymb,multirow,bigstrut,booktabs,ctable,latexsym}
\usepackage{mathtools}
\usepackage[mathscr]{eucal}
\usepackage{verbatim}

\usepackage{amsthm}
\usepackage{algorithm}
\usepackage[normalem]{ulem}

\newtheorem{theorem}{Theorem}

\newtheorem{prop}{Proposition}
\newtheorem{lemma}{Lemma}

\newcommand{\argmin}{\arg\!\min}

\IEEEoverridecommandlockouts

\title{\Large Risk-Aware Distributed Multi-Agent Reinforcement Learning}
	
	\author{Abdullah Al Maruf$^{1}$, Luyao Niu$^{1}$, Bhaskar Ramasubramanian$^{2}$, Andrew Clark$^{3}$, Radha Poovendran$^{1}$%
		\thanks{$^{1}$Network Security Lab, Department of Electrical and Computer Engineering, 
			University of Washington, Seattle, WA 98195, USA. 
			{\tt\small \{maruf3e, luyaoniu, rp3\}@uw.edu}}
   \thanks{$^{2}$Electrical and Computer Engineering, Western Washington University, Bellingham, WA 98225, USA.
			{\tt ramasub@wwu.edu}}
   \thanks{$^{3}$Electrical and Systems Engineering, Washington University in St. Louis, St. Louis, MO 63130, USA. 
   {\tt\small andrewclark@wustl.edu}
   }
	}
\begin{document}	
\maketitle 

\begin{abstract}
Autonomous cyber and cyber-physical systems need to perform decision-making, learning, and control in unknown environments. 
Such decision-making can be sensitive to multiple factors, including modeling errors, changes in costs, and impacts of events in the tails of probability distributions. 
Although multi-agent reinforcement learning (MARL) provides a framework for learning behaviors through repeated interactions with the environment by minimizing an average cost, it will not be adequate to overcome the above challenges. 
In this paper, we develop a distributed MARL approach to solve decision-making problems in unknown environments by learning risk-aware actions. 
We use the conditional value-at-risk (CVaR) to characterize the cost function that is being minimized, and define a Bellman operator to characterize the value function associated to a given state-action pair. 
We prove that this operator satisfies a contraction property, and that it converges to the optimal value function. 
We then propose a distributed MARL algorithm called the \emph{CVaR QD-Learning} algorithm, and establish that value functions of individual agents reaches consensus. 
We identify several challenges that arise in the implementation of the \emph{CVaR QD-Learning} algorithm, and present solutions to overcome these. 
We evaluate the \emph{CVaR QD-Learning} algorithm through simulations, and demonstrate the effect of a risk parameter on value functions at consensus.
\end{abstract}


\section{Introduction}\label{sec:Introduction}

Reasoning about the satisfaction of objectives for complex cyber and cyber physical systems typically involves solving a sequential decision-making problem. 
The operating environment is represented as a Markov decision process (MDP)~\cite{puterman2014markov}, and transitions between any two states in the system is a probabilistic outcome based on the actions of the decision maker or agent. 
In dynamic and uncertain environments, the frameworks of reinforcement learning~\cite{sutton2018reinforcement} and optimal control~\cite{bertsekas2017dynamic} have been used to solve sequential decision-making problems by determining actions to minimize an accumulated cost. 
\emph{Risk-neutral} decision making solutions determine actions by minimizing an expected or average cost; such techniques have been implemented in applications including robotics, mobile networks, and games~\cite{hafner2011reinforcement, mnih2015human, silver2016mastering, zhang2019deep, sadigh2016planning, yan2018data, you2019advanced}. 

Although risk-neutral solutions are computationally tractable, they have been shown to have limitations in characterizing sensitivity to changes in costs, modeling errors, and the effect of tails of probability distributions~\cite{howard1972risk, nilim2005robust, borkar2002q}. 
In order to solve a sequential decision-making problem by learning \emph{risk-aware} actions, we use the \emph{conditional value-at-risk (CVaR)}~\cite{rockafellar2002conditional, ahmadi2021constrained, chapman2021risk, lindemann2020control} to characterize the objective function of an MDP. 
The CVaR corresponds to the average value of the cost conditioned on the event that the cost takes sufficiently large values, and was shown to have
strong theoretical justification for its use in~\cite{rockafellar2002conditional, chow2015risk}.  
Optimizing a CVaR-based cost will ensure sensitivity of actions to rare high-consequence outcomes~\cite{serraino2013conditional}. 
However, different from~\cite{chow2015risk}, we assume that the operating environment of the agent is unknown. 
The agent then learns behaviors by minimizing a cost that is revealed through repeated interactions with the environment. 
For this setting, we develop a CVaR-based variant of the classical Q-learning algorithm~\cite{sutton2018reinforcement}, and establish its convergence. 

When multiple decision makers share the same environment, each agent interacts with both, the environment and with other agents. 
Consequently, the evolution of agents' behaviors has been shown to be non-stationary from the perspective of any single agent~\cite{tan1993multi, matignon2012independent, zhang2021multi}. 
The literature examining incorporation of risk-sensitivity in MARL is limited. 
Recently, the authors of \cite{qiu2021rmix, zhao2022dqmix} developed a framework to learn risk-sensitive policies in cooperative MARL using CVaR. 
The algorithms proposed in the above works use the centralized training with decentralized execution (CTDE) paradigm~\cite{lowe2017multi} to learn behaviors. 
An agent using CTDE can use information about other agents’ observations and actions to aid its own learning during training, but will have to take decisions independently at test-time~\cite{foerster2018counterfactual, rashid2020monotonic}.  

Different from the above works, in this paper, we design a distributed \emph{risk-aware multi-agent reinforcement learning} algorithm. 
Our solution is inspired by QD-learning~\cite{kar2013cal}, wherein at each step, a single update rule incorporates costs revealed by the environment and information from neighboring agents in a graph that describes inter-agent communication. 
We establish the consensus 
of agents' value functions when they are optimizing a CVaR-based cost. 
Our experiments also reveal that as agents become more risk-aware, their value functions at consensus increase in magnitude (corresponding to higher costs incurred); this observation agrees with intuition when the goal is to minimize an accumulated cost. 
We make the following specific contributions: 

\begin{itemize}
    \item We define a Bellman operator to characterize a CVaR-based state-action value function. 
    \item We prove that the Bellman operator is a contraction, and that the fixed point of the operator is the optimal risk-aware value function. 
    \item We develop a risk-aware distributed multi-agent reinforcement learning algorithm called \emph{CVaR QD-Learning} and prove that CVaR-based value functions of individual agents reaches consensus. 
    \item We carry out experimental evaluations to validate the \emph{CVaR QD-Learning} algorithm, and show that value functions at consensus increase in magnitude as agents become more risk-aware. 
\end{itemize}

The remainder of this paper is organized as follows: 
Sec. \ref{sec:Preliminaries} introduces necessary preliminaries on reinforcement learning and risk-aware decision making. 
Sec. \ref{sec:Results1} presents construction of a Bellman operator, and shows that it is a contraction. 
Sec. \ref{sec:Results2} presents CVaR-based QD-learning and associated analytical results, and Sec. \ref{CVaRAlgo} presents the \emph{CVaR QD-Learning} algorithm and describes how challenges in the implementation of the algorithm are overcome. 
Sec. \ref{sec:Experiments} shows results of experimental evaluations and Sec. \ref{sec:Conclusion} concludes the paper. 
\section{Setup and Problem Formulation} \label{sec:Preliminaries}
This section introduces the Markov game setup that we consider, provides necessary preliminaries on reinforcement learning and risk criteria used in decision-making. We then formally state the problem that we will solve in this paper. 

\subsection{Setup}
We consider a system with $N$ agents. Inter-agent communication is described by an undirected graph $\mathcal{G}=(\mathcal{V}, \mathcal{E})$, where $\mathcal{V}=\{1, \cdots, N\}$ is the set of vertices (or nodes)  and $\mathcal{E} \subset \mathcal{V} \times \mathcal{V}$ is the set of edges between pairs of vertices. Here the nodes and the edges in the graph $\mathcal{G}$ correspond to the agents and the communication link between agents. We assume that $\mathcal{G}$ is simple (no self-loops or multiple edges between any two vertices) and connected (there is a path between every pair of nodes). The set of neighbors of agent $n$ is denoted by $\mathcal{N}(n)$. 
The graph can be described by an $N\times N$ Laplacian matrix $L$ with entries $L_{ij}=-1$ if $(i,j) \in \mathcal{E}$ or otherwise zero, and $L_{ii}=|\mathcal{N}(i)|$ which is equal to the degree of node $i$. 
Since $\mathcal{G}$ is connected, the eigenvalues of $L$ can be ordered as $0=\lambda_1(L) < \lambda_2(L) \leq \cdots \leq \lambda_N(L)$ \cite{chung1997spectral}.

The multi-agent setup we consider here is similar to that of \cite{kar2013cal}. We assume that each agent can fully observe the state of the system.  
However, the cost/reward received by each agent is local and may not be available to a remotely located controller. 
As an example, this multi-agent setup can resemble spatially distributed temperature sensors (agents) in a 
building~\cite{kar2013cal}. A remote controller will have access to all sensor readings but will not be aware of the desired temperatures at different rooms in the building. 
A possible objective of the controller could be to minimize the average squared deviation between measured temperatures from sensors and their corresponding locations' desired temperatures. 

In such a setting, behaviors of agents in the environment can be described by a Markov game $M := (\mathcal{S}, \mathcal{A}, \{c_1, \cdots c_N\},P, \gamma)$, where $\mathcal{S}$ and $\mathcal{A}$ are assumed to be finite state and action spaces. When the system is in state $s$ and the action taken by the controller is $a$, the agent $n$ incurs a bounded and deterministic local cost $c_n(s,a) \in [-C_{max},C_{max}]$. We emphasize that in our setup the state $s \in \mathcal{S}$ and $a \in \mathcal{A}$ are global (i.e. common to all agents) whereas individual costs $c_n(s,a)$ are local to each agent. $P(s'|s,a)$ gives the probability of transitioning to state $s' \in \mathcal{S}$ when taking action $a \in \mathcal{A}$ in state $s \in \mathcal{S}$, and $\gamma \in [0,1)$ is the discounting factor. However, different from \cite{kar2013cal}, we assume that costs are deterministic to aid the development of our theoretical results.

A trajectory of $M$ is an alternating sequence of states and actions $(s_0,a_0,s_1,a_1,\dots)$. A history up to time $k$, denoted as $h_k 
\in H_k$, corresponds to a trajectory up to time $k$ i.e.  $(s_0,a_0,s_1,a_1,\dots, s_k)$. Formally, with $H_0=S$, we recursively define the set of possible histories up to time $k \geq 1$ as $H_k=H_{k-1} \times \mathcal{A} \times \mathcal{S}$. 
A policy at time $k$ is a map $\Pi_{H,k}: H_k \rightarrow A$, and we define $\Pi_H := \lim_{k\rightarrow \infty} \Pi_{H,k}$ to be the set of all history-dependent policies. A policy $\mu(s_k)$ is called Markov when it only depends on the current state $s_k$.

Let $c(s_k,a_k)= \frac{1}{N} \sum_{n=1}^N c_n(s_k,a_k)$ be the average costs over all agents observed at time $k$. 
The discounted average cost up to time $k$ is defined as $C_{0,k}:= \sum_{t=0}^k  \gamma^t c(s_t,a_t) = \frac{1}{N} \sum_{t=0}^k  \gamma^t \sum_{n=1}^N c_n(s_t,a_t)$. Thus discounted average cost over the infinite horizon is given by  $C_{0,\infty}=\lim_{k \rightarrow \infty} C_{0,k}$. 

In reinforcement learning (RL), the transition probability $P$ is not known, and costs are revealed to agents through repeated interactions with the environment \cite{sutton2018reinforcement}.  
The objective is to learn a policy that minimizes the expected accumulated discounted cost. One widely-used approach to learn such a policy is the Q-learning algorithm \cite{watkins1992q}. 
Using a state-action value function $Q^\pi(s,a):= \mathbb{E}_{\pi}[C_{0,\infty}|s_0=s,a_0=a, \pi]$, the Q-learning algorithm seeks to find the optimal value $Q^*(s,a)$ corresponding to the optimal policy $\pi^*$ such that $Q^*(s,a) \leq Q^\pi(s,a)$ for all $(s,a) \in \mathcal{S} \times \mathcal{A}$ and any policy $\pi$.

\subsection{QD-Learning}

QD-learning is a multi-agent distributed variant of Q-learning when agent behaviors in unknown environments were also influenced an underlying communication graph $\mathcal{G}$ that was first proposed in \cite{kar2013cal}. 
In the QD-learning algorithm \cite{kar2013cal}, at time-step $k$, each agent $n$ maintains a sequence of state-action value functions $\{Q_{n,k}(s,a)\} \in \mathbb{R}^{|\mathcal{S}\times \mathcal{A}|}$  for all state-action pairs $(s,a) \in \mathcal{S}\times \mathcal{A}$. The sequence $\{Q_{n,k}(s,a)\}$ is updated according to the following rule \cite{kar2013cal}:
\begin{align} \label{qd_l_update_orig}
&Q_{n,k+1}(s_k,a_k)= (1-\alpha_k)  Q_{n,k}(s_k,a_k) \nonumber \\
&\quad \qquad + \alpha_k \big(c_n(s_k,a_k) + \gamma  \min_{a' \in \mathcal{A}} Q_{n,k}(s_{k+1},a') \big)  \nonumber \\
&\qquad \qquad  +\beta_k \sum_{l \in \mathcal{N}(n)} \big(Q_{n,k}(s_k,a_k)- Q_{l,k}(s_k,a_k)\big),
\end{align}
where the weight sequences $\{\alpha_k\}$ and $\{\beta_k\}$ are given by 
\begin{align} 
& \alpha_k = \frac{a}{(k+1)^{\tau_1}}, \label{eq:alpha_orig} \\
& \beta_k = \frac{b}{(k+1)^{\tau_2}},\label{eq:beta_orig}
\end{align}
with $a$ and $b$ being positive constants. Equations \eqref{eq:alpha_orig} and \eqref{eq:beta_orig} guarantee that the excitation for the innovation and consensus terms in Eqn. \eqref{qd_l_update_orig} are persistent; i.e., $\sum_k \alpha_k = \infty$ and $\sum_k \beta_k = \infty$. The sequences $\{\alpha_k\}$ and $\{\beta_k\}$ further satisfy $\sum_k \alpha_k^2 < \infty$, $\sum_k \beta_k^2 < \infty$ and $\frac{\beta_k}{\alpha_k} \rightarrow \infty$ as $k \rightarrow \infty$. Constants $a$ and $b$ in Eqns. \eqref{eq:alpha_orig} and \eqref{eq:beta_orig} are also chosen so that $(I_N- \beta_k L- \alpha_k I_N)$ is positive semidefinite for all $k$ where $L$ is the Laplacian matrix of graph $\mathcal{G}$ and $I_N$ denotes the identity matrix in $\mathbb{R}^{N \times N}$. One way to ensure such positive semidefiniteness is to choose $a$ and $b$ so that $a+Nb \leq 1$. We refer the reader to \cite{kar2013cal} for a detailed discussion on the selection of the weight sequence. When the weight sequences are chosen appropriately, the authors of \cite{kar2013cal} showed that the sequence $\{Q_{n,k}(s,a)\}$ asymptotically reaches consensus for all agents and the value achieved at consensus was equal to the optimal action-value function $Q^*(s,a)$.

\subsection{Conditional Value-at-Risk (CVaR)}
Let $Z$ be a bounded random variable on the probability space $(\Omega, \mathcal{F},P)$ with cumulative distribution $F(z)=P(Z\leq z)$. The value-at-risk (VaR) at confidence level $y \in (0,1]$ is 
$VaR_{y}(Z):= min \{z| F(z) \geq 1- y \}$~\cite{rockafellar2002conditional}. The conditional value-at-risk (CVaR) at confidence level $y$ is defined as $CVaR_{y}(Z):= \mathbb{E}[Z|Z \geq VaR_{y}(Z)]$~\cite{rockafellar2002conditional}, and represents the expected value of $Z$, conditioned on the $y$ quantile of the tail distribution. We note that $CVaR_{y}(Z)=\mathbb{E}(Z)$ when $y=1$ and $CVaR_{y}(Z) \rightarrow max\{Z\}$ as $y \rightarrow 0$. 
The CVaR of a random variable $Z$ can be interpreted as the worst-case expectation of $Z$ under a perturbed distribution $\xi P$, as given by the following result from \cite{chow2015risk,artzner1999coherent}. 
\begin{prop}[Dual CVaR formulation \cite{chow2015risk,artzner1999coherent}]\label{DualRep}
Let $\mathbb{E}_{\xi}[Z]$ denote the $\xi$-weighted expectation of $Z$ and $\mathcal{U}_{CVaR}(y,P):=\{ \xi: \xi(\omega) \in [0,\frac{1}{y}], \int_{\omega \in \Omega} \xi(\omega) P(\omega) d\omega=1\}$. Then, 
\begin{align*}
    CVaR_{y}(Z)&=\max_{\xi \in \mathcal{U}_{CVaR}(y,P)}\mathbb{E}_{\xi}[Z].
\end{align*}
\end{prop}

The above dual representation, together with the coherent property of the CVaR metric was used to derive a decomposition of $CVaR_{y}(Z)$ in a recursive manner \cite{chow2015risk,pflug2016time}. 
In our setting, the random variable $Z$ is a sequence of costs. 

\begin{prop} [CVaR decomposition \cite{chow2015risk}]\label{CVaR_decomp}
For $k\geq 0$, let $Z=\big(Z_{k+1}, Z_{k+2}, \cdots)$ denote the sequence of costs starting from time $k+1$ to onward. Then, the conditional CVaR, under a policy $\pi$ satisfies 
\begin{align*}
    &CVaR_{y}(Z|h_k, \pi)\\&=\max_{\xi}~\mathbb{E}[\xi(s_{k+1})CVaR_{y \xi(s_{k+1})}(Z|h_{k+1},\pi)|h_k,\pi], 
\end{align*}
where $\xi \in \mathcal{U}_{CVaR}(y,P(\cdot|s_t,a_t))$, and the expectation is with respect to $s_{k+1}$. 
\end{prop}

\subsection{Problem Formulation}
Our objective in this paper is to learn risk-aware policies for multi-agent reinforcement learning in a distributed manner. 
We use CVaR as a measure of risk sensitivity, and aim to learn a policy $\pi \in \Pi_H$ that will minimize a risk-sensitive discounted cost. 
We use the terms \emph{confidence level} and \emph{risk parameter} interchangeably to refer to the parameter $y \in (0,1]$. 
The challenge in solving this problem is that the transition probabilities are not known, and costs associated to taking an action in a particular state are revealed to agents only through repeated interactions with the environment. We formally state the problem that we want to address as: 
\begin{align} \label{prob}
 \min_{\pi \in \Pi_H} CVaR_{y}\big(C_{0,\infty}|s_0,\pi\big). 
\end{align}
In order to solve the problem given in Eqn. \eqref{prob}, we will first propose a CVaR-based Bellman operator for Q-learning and show that the operator has a fixed point that corresponds to the optimal solution. Then we will design a CVaR-based QD-learning algorithm and show that the value functions of individual agents reach a consensus. Through simulation, we also verify the convergence of the algorithm. 
\section{Bellman Operator Construction}\label{sec:Results1}

In this section, we define a Bellman operator in order to compute the optimal state-action value function for CVaR-based reinforcement learning. 
Similar to techniques used for other variants of Q-learning, e.g., \cite{watkins1992q, borkar2002q, borkar2021prospect, ramasubramanian2021reinforcement}, we 
establish the convergence of this operator by showing that it satisfies a contraction property. 
Of particular relevance is the result for CVaR-based dynamic programming for the convergence of state value functions of MDPs (whose transition probabilities and cost structures are known apriori) presented in \cite{chow2015risk}. 

Leveraging Proposition \ref{CVaR_decomp}, we first augment the state-action pair $(s,a)$ with an additional `continuous state' $y \in \mathcal{Y}=(0,1]$ which represents the confidence level for CVaR.  Then, for a given policy $\pi$ and CVaR confidence level $y \in \mathcal{Y}$, we define the augmented state-action value function as: 
\begin{align}\label{CVaRQFunc}
    Q^\pi(s,a,y)&:=CVaR_y(C_{0,\infty}|s_0=s,a_0=a,\pi).
\end{align}
To set up a dynamic programming characterization of $Q^\pi(s,a,y)$, we define a Bellman operator on the space of augmented state-action value functions. 
Consider the \emph{CVaR Bellman operator} $\mathbf{T}: \mathcal{S} \times \mathcal{A} \times \mathcal{Y} \rightarrow \mathcal{S} \times \mathcal{A} \times \mathcal{Y}$ and $\xi \in \mathcal{U}_{CVaR}(y,P(\cdot|s,a))$. Then, we define 
\begin{align} \label{Bellman_op}
\mathbf{T}[Q](s,a,y) := c(s,a)+ & \gamma \min_{a' \in \mathcal{A}} \Big[ \max_{\xi} \sum_{s' \in \mathcal{S}} \xi(s') \nonumber \\
& P(s'|s,a) Q(s',a', y\xi(s')) \Big].
\end{align}

Our first result formalizes the fact that the Bellman operator defined in Eqn. \eqref{Bellman_op} satisfies a contraction property. 


\begin{lemma} [Contraction]\label{lemma: contr}
Let $Q^1(s,a,y)$ and $Q^2(s,a,y)$ be two augmented state-action value functions for the same $(s,a,y)$ as defined in Eqn. (\ref{CVaRQFunc}). Then, the Bellman operator $\mathbf{T}$ defined in Eqn. \eqref{Bellman_op} is a contraction under the $sup$-norm. That is, 
$||\mathbf{T}[Q^1](s,a,y)-\mathbf{T}[Q^2](s,a,y)||_{\infty} \leq \gamma ||Q^1(s,a,y)-Q^2(s,a,y)||_{\infty}$.
\end{lemma}

\begin{proof}
Our proof uses an argument similar to \cite{chow2015risk}. 
However, different from \cite{chow2015risk}, the environment in our setting is unknown; this will require reasoning about state-action value functions~\cite{sutton2018reinforcement}, rather than state value functions used in \cite{chow2015risk}. 

Since $\xi(s') P(s'|s,a') \geq 0$ for all $\xi \in \mathcal{U}_{CVaR}(y,P(\cdot|s,a))$, we have $\mathbf{T}[Q^1](s,a,y) \leq \mathbf{T}[Q^2](s,a,y)$ whenever $Q^1(s,a,y) \leq Q^2(s,a,y)$ for all $s \in \mathcal{S}, a \in \mathcal{A}, y \in \mathcal{Y}$. 
Moreover, since $ \sum_{s' \in \mathcal{S}} \xi(s') P(s'|s,a') =1$ for all $\xi \in \mathcal{U}_{CVaR}(y,P(\cdot|s,a))$, we have $\mathbf{T}[Q+c](s,a,y)= c \gamma+ \mathbf{T}[Q](s,a,y)$ where $c$ is a constant. Thus, we have established that $\mathbf{T}[Q](s,a,y)$ exhibits the monotonicity and constant shift properties. 
Now, from the definition of sup-norm, for all
$s \in \mathcal{S}, a \in \mathcal{A}, y \in \mathcal{Y}$ we have $|Q^1(s,a,y) - Q^2(s,a,y)| \leq ||Q^1(s,a,y) - Q^2(s,a,y)||_{\infty}$ which is equivalent to
\begin{align} \label{ineq}
 &Q^2(s,a,y) -||Q^1(s,a,y) - Q^2(s,a,y)||_{\infty} \leq Q^1(s,a,y) \nonumber \\ 
 &\leq Q^2(s,a,y)+||Q^1(s,a,y) - Q^2(s,a,y)||_{\infty}.   
\end{align}
Applying the Bellman operator to the terms in Eqn. \eqref{ineq} and leveraging the monotonicity and constant shift properties with $c=\pm ||Q^1(s,a,y) - Q^2(s,a,y)||_{\infty}$, for all $s \in \mathcal{S}, a \in \mathcal{A}, y \in \mathcal{Y}$, we obtain 
\begin{align} \label{ineq2}
&\mathbf{T}[Q^2](s,a,y)- \gamma||Q^1(s,a,y) - Q^2(s,a,y)||_{\infty} \leq \nonumber \\
&~~~~~~~~~~~\mathbf{T}[Q]^1(s,a,y)\leq \mathbf{T}[Q]^2(s,a,y) \nonumber \\
&~~~~~~~~~~~~~~~+\gamma||Q^1(s,a,y) - Q^2(s,a,y)||_{\infty}.
\end{align}
This is equivalent to $|\mathbf{T}[Q^1](s,a,y)-\mathbf{T}[Q^2](s,a,y)| \leq \gamma ||Q^1(s,a,y) - Q^2(s,a,y)||_{\infty}$ $\forall s \in \mathcal{S}, a \in \mathcal{A}, y \in \mathcal{Y}$, which completes the proof.
\end{proof}

Lemma \ref{lemma: contr} allows us to establish the convergence of the CVar Bellman operator $\mathbf{T}$ to a fixed point $Q^f(s,a,y)$ when applied repeatedly. Let $Q^*(s,a,y)$ be the optimal state-action value function with respect to CVaR confidence level $y$ i.e. 
\begin{align} \label{eq:Q*}
 Q^*(s,a,y)= \min_{\pi \in \Pi_H} CVaR_y(\mathcal{C}_{0,\infty}|x_0=s, a_0=a, \pi).   
\end{align}
We now derive an intermediate result that will allow us to show the optimality of the fixed point of CVaR Bellman operator i.e. $Q^f(s,a,y)= Q^*(s,a,y)$. To do so, by $\mathbf{T}^k$ we denote the application of the Bellman operator $\mathbf{T}$ for $k$ times for some $k \in \mathbb{N}$.

\begin{lemma} \label{lemma:optimal}
Let $Q_k(s,a,y):=\mathbf{T}^k[Q_0](s,a,y)$ for all $s \in \mathcal{S}, a \in \mathcal{A}, y \in \mathcal{Y}$ where $k \in \mathbb{N}$ and $Q_0$ is an arbitrary initial value of $Q(s,a,y)$ for all $s \in \mathcal{S}, a \in \mathcal{A}, y \in \mathcal{Y}$. Then $Q_k(s,a,y)= \min_{\pi \in \Pi_k} CVaR_y(\mathcal{C}_{0,k-1}+ \gamma^k Q_0|s_0=s,a_0=a,\pi)$.
\end{lemma}

\begin{proof} We prove this result using induction. For the base case i.e. $k=1$, we have 
\begin{align}
 & Q_1(s,a,y) =\mathbf{T}[Q_0](s,a,y) = c(s,a) + \gamma \min_{a' \in \mathcal{A}} \nonumber \\ 
 &~~\Big[\max_{\xi \in \mathcal{U}_{CVaR}(y,P(\cdot|s,a))} \sum_{s' \in \mathcal{S}} \xi(s') P(s'|s,a)  Q_0(s',a', y\xi(s')) \Big]   \nonumber \\
 &= c(s,a)+ \min_{a' \in \mathcal{A}} \Big[ \gamma \max_{\xi \in \mathcal{U}_{CVaR}(y,P(\cdot|s,a))} \mathbb{E}_{\xi(s')}[Q_0 \nonumber \\
 &~~~~~~~~~~~~~~~~~~~~~~~~~~|s_0=s',a_0=a',\pi]\Big] \nonumber  \\
 &=\min_{a' \in \mathcal{A}} \Big[c(s,a)+ \gamma \max_{\xi \in \mathcal{U}_{CVaR}(y,P(\cdot|s,a))} \mathbb{E}_{\xi(s')} [CVaR_{y\xi(s')} \nonumber \\
& ~~~~~~~~~~~~~~~~~~~~~~~~~~~~~~(Q_0|s',a')| s_0=s,a_0=a,\pi)]\Big] \nonumber \\
 &= \min_{a' \in \mathcal{A}} \Big[c(s,a)+ \gamma CVaR_y(Q_0|s_0=s,a_0=a,\pi) \Big] \nonumber \\
  &= \min_{\pi \in \pi_1} CVaR_y(\mathcal{C}_{0,0}+ \gamma Q_0|s_0=s,a_0=a,\pi). \nonumber 
\end{align}
 Thus the base case is proved. Now we assume that the result holds for $k=i$. Now for $k=i+1$ we have 
\begin{align}
 & Q_{i+1}(s,a,y) = \mathbf{T}^{i+1}[Q_0](s,a,y) = \mathbf{T}[Q_i](s,a,y) \nonumber \\
 &= c(s,a) + \gamma \min_{a' \in \mathcal{A}} 
\Big[ \max_{\xi \in \mathcal{U}_{CVaR}(y,P(\cdot|s,a))}\sum_{s' \in \mathcal{S}} \xi(s') P(s'|s,a) \nonumber \\
&~~~~~~~~~~~~~~~~~~~~~~~~~~~~~~Q_i(s',a', y\xi(s')) \Big]   \nonumber \\
 &= c(s,a) + \gamma \min_{a' \in \mathcal{A}} 
\Big[ \max_{\xi \in \mathcal{U}_{CVaR}(y,P(\cdot|s,a))}\sum_{s' \in \mathcal{S}} \xi(s') P(s'|s,a) \nonumber \\
&~~~~~~\min_{\pi \in \Pi_i} CVaR_{y\xi(s')}(\mathcal{C}_{0,i-1}+ \gamma^i Q_0|s_0=s',a_0=a',\pi) \Big]   \nonumber \\
 &=\min_{a' \in \mathcal{A}} \Big[c(s,a)+ \max_{\xi \in \mathcal{U}_{CVaR}(y,P(\cdot|s,a))} \mathbb{E}_{\xi(s')} [\min_{\pi \in \Pi_i} \nonumber \\
 &~~~~~~ CVaR_{y\xi(s')}(\mathcal{C}_{1,i}+ \gamma^{i+1} Q_0|s',a') |s_0=s,a_0=a,\pi)]\Big] \nonumber \\
  &= \min_{\pi \in \pi_{i+1}} CVaR_y(\mathcal{C}_{0,i}+ \gamma^{i+1} Q_0|s_0=s,a_0=a,\pi).
\end{align}
Hence the proof is complete by induction.
\end{proof}

The main result of this section proves that the fixed point of the Bellman operator $\mathbf{T}[Q](s,a,y)$ is the optimal value function $Q^*(s,a,y)$.

\begin{theorem} \label{Q_MDP}
As $k \rightarrow \infty$, $Q_k(s,a,y)$ converges to a unique fixed point $\mathbf{T}[Q^*](s,a,y)= Q^*(s,a,y)$ for all $s \in \mathcal{S}$, $a \in \mathcal{A}$ and $y \in \mathcal{Y}$.
\end{theorem}

\begin{proof}
The proof follows from Lemma \ref{lemma: contr} and Lemma \ref{lemma:optimal}. Since the CVaR Bellman operator $\mathbf{T}[Q](s,a,y)$ has the contraction property, we have that $\lim_{k \rightarrow \infty} Q_k(s,a,y)$ converges to a unique fixed point $\mathbf{T}[Q^f](s,a,y)= Q^f(s,a,y)$. 
From Lemma \ref{lemma:optimal} and using the fact that  $\gamma < 1$, we can write
\begin{align*}
 & Q^f(s,a,y)\\
 &= \lim_{k \rightarrow \infty} Q_k(s,a,y) \\
 &= \min_{\pi \in \Pi_H} CVaR_y(\lim_{k \rightarrow \infty} \mathcal{C}_{0,k-1} +\gamma^k Q_0|x_0=s, a_0=a, \pi) \\
&= \min_{\pi \in \Pi_H} CVaR_y(\mathcal{C}_{0,\infty}|x_0=s, a_0=a, \pi) \\
&=Q^*(s,a,y).
\end{align*}
Hence the proof is complete.
\end{proof}

We note that the optimal value function $Q^*(s,a,y)$ can be achieved by a stationary (deterministic) Markov policy $\mu^*(s,y)$ where 
$\mu^*$ is a mapping from the current state $s$ and CvaR confidence level $y$ to an action $a$, i.e., $\mu^*: (\mathcal{S},\mathcal{Y}) \rightarrow \mathcal{A}$. 
We obtain the fixed point $Q^*(s,a,y)$ through repeated application of the CVaR Bellman operator, and then determine  
a greedy policy $\mu^*$ on $Q^*(s,a,y)$ such that $\mu^*(s,y)= \argmin_a Q^*(s,a,y)$ \cite{kar2013cal}. Then, from the definition of $Q^*(s,a,y)$ in Eqn. \eqref{eq:Q*}, it follows that the greedy policy $\mu^*$ indeed achieves optimality. 
This makes the problem of finding an optimal policy tractable, even though the original problem in (\ref{prob}) is defined over the set of history-dependent policies.

\section{CVaR-based Distributed Multi-Agent Reinforcement Learning}\label{sec:Results2}

We use the insight from Theorem \ref{Q_MDP} to guide the design of a distributed multi-agent reinforcement learning algorithm where agents seek to optimize a risk-aware objective. 
Such an objective is expressed in terms of a CVaR-based discounted accumulated cost (described in Eqn. \eqref{CVaRQFunc}). 
The update rule for our \emph{CVaR QD-Learning} algorithm is informed by the QD-learning scheme proposed in \cite{kar2013cal}. 
However, optimizing a CVaR-based value function (instead of an expectation-based function in \cite{kar2013cal}) will require us to develop additional mathematical structure, which will be described below. 

At each iteration of the \emph{CVaR QD-Learning} update, the augmented state-action value function of an agent evolves as a weighted sum of (i) the augmented state-action value function at the previous iteration, (ii) an \emph{innovation} term arising from the cost observed by taking a particular action in a given state, and (iii) a \emph{consensus} term corresponding to the difference between augmented state-action values of the agent with its neighbors $\mathcal{N}(\cdot)$ in the inter-agent communication graph $\mathcal{G}$. 
Specifically, the sequence $\{Q_{n,k}(s,a,y)\}$ evolves for each agent $n$ according to the following equation: 
\begin{align} \label{qd_l_update}
&Q_{n,k+1}(s_k,a_k,y_k)= (1-\alpha_k)  Q_{n,k}(s_k,a_k,y_k) \nonumber \\
&~~~+ \alpha_k \Big(c_n(s_k,a_k) + \gamma  \min_{a' \in \mathcal{A}} \big[ \max_{\xi \in \mathcal{U}_{CVaR}(y,P(\cdot|s,a))} \nonumber \\
&~~~~~~~~~~~~~~~~~~~~~~~~~~\xi(s_{k+1}) Q_{n,k}(s_{k+1},a',y_k  \xi(s_{k+1})) \big] \Big)  \nonumber \\
&-\beta_k \sum_{l \in \mathcal{N}(n)} \big(Q_{n,k}(s_k,a_k,y_k)- Q_{l,k}(s_k,a_k,y_k)\big),
\end{align}
where the weight sequences $\{\alpha_k\}$ and $\{\beta_k\}$ are given by \eqref{eq:alpha_orig} and \eqref{eq:beta_orig} where $\tau_1 \in (\frac{1}{2},1]$ and $\tau_2 \in (0,\tau_1-\frac{1}{2})$. Like \cite{kar2013cal}, here the weight sequences satisfy  $\sum_k \alpha_k = \infty$, $\sum_k \beta_k = \infty$, $\sum_k \alpha_k^2 < \infty$, $\sum_k \beta_k^2 < \infty$ and $\frac{\beta_k}{\alpha_k} \rightarrow \infty$ as $k \rightarrow \infty$. Constants $a$ and $b$ in Eqns. \eqref{eq:alpha_orig} and \eqref{eq:beta_orig} also satisfy that $(I_N- \beta_k L- \alpha_k I_N)$ is positive semidefinite for all $k$.

Since the costs $\{c_n(s,a)\}$ and the parameter $\{\xi(s)\}$ are bounded, and the initial augmented state-action value functions $\{Q_{n,0}(s,a,y)\}$ are chosen to be bounded for all agents $n$, and for all $s \in \mathcal{S}, a \in \mathcal{A}$ and $y \in \mathcal{Y}$, we can show that 
$\{Q_{n,k}(s_k,a_k,y_k)\}$ is pathwise bounded, i.e., $P(\sup_k ||Q_{n,k}||_{\infty} < \infty)=1$. 
Our next result establishes that the augmented state-action value functions of the agents asymptotically reach a consensus. 

\begin{theorem}[Consensus of CVaR QD-Learning] \label{CVaRQDConsensus}
 For the CVaR QD-learning update in Eqn. \eqref{qd_l_update}, each agent reaches consensus asymptotically for all $s \in \mathcal{S}, a \in \mathcal{A}$ and $y \in \mathcal{Y}$. That is,  
\begin{align} \label{qdl_eq1}
P(lim_{k \rightarrow \infty} ||Q_{n,k}(s,a,y)-\bar{Q}_k(s,a,y)||=0)=1, 
\end{align}
where 
\begin{align}
\bar{Q}_k(s,a,y)= \frac{1}{N} \sum_{n=1}^N Q_{n,k}(s,a,y).
 \end{align} 
\end{theorem}

\begin{proof}
For agent $n$ we can write the update in Eqn. \eqref{qd_l_update} as:
\begin{align} \label{qd_l_proof1}
z_{k+1}= (I_N- \beta_k L- \alpha_k I_N) z_k + \alpha_k(U_{k}+J_{k})
\end{align}
where $L$ is the Laplacian matrix of the graph $\mathcal{G}$ and
\begin{align*}
 z_k &:= [Q_{1,k}(s_k=s,a_k=a,y_k=a) \cdots \\
&~~~~~~~~~~~~~~~~~~~~~Q_{n,k}(s_k=s,a_k=a,y_k=y)]^T, \\
U_{k} &:= \gamma  \min_{a' \in \mathcal{A}} \big[ \max_{\xi \in \mathcal{U}_{CVaR}(y,P(\cdot|s,a))} \xi(s_{k+1}=s')\\ &~~~~~~~~~~~~~~~~~~~~~~~~~~~~Q_{n,k}(s_{k+1}=s',a', y_k  \xi(s_{k+1})) \big], \\
J_{k} &:= c_n(s_k=s,a_k=a).   
\end{align*} 
Let $\hat{z}_k= z_k - \frac{1}{N} 1_N^T z_k 1_N$ where $1_N$ is a vector in $\mathbb{R}^N$ with all entry as $1$. Then we can write Eqn. \eqref{qd_l_proof1} as
\begin{align} \label{qd_l_proof2}
\hat{z}_{k+1}= (I_N- \beta_k L- \alpha_k I_N) \hat{z}_k + \alpha_k(\hat{U}_{k}+\hat{J}_{k})
\end{align}
where $\hat{U}_{k}= U_k -\frac{1}{N} 1_N^T U_k 1_N$ and $\hat{J}_{k}= U_k -\frac{1}{N} 1_N^T J_k 1_N$. Then following the argument of \cite{kar2013cal} and applying Lemma 4.2 in \cite{kar2013cal}, we can write for $k \geq k_0$
\begin{align} \label{qd_l_proof3}
||(I_N- \beta_k L- \alpha_k I_N) \hat{z}_k|| \leq (1-c_2~r_k)||\hat{z}_k||
\end{align}
where $c_2 \in (0,1)$, $0 \leq {r_k} \leq 1$ and $k_0\in \mathbb{N}$. Combining Eqns. \eqref{qd_l_proof2} and \eqref{qd_l_proof3}, we can write for $k \geq k_0$
\begin{align} \label{qd_l_proof4}
||\hat{z}_{k+1}|| \leq (1-c_2~r_k)||\hat{z}_k|| +\alpha_k(||\hat{U}_k||+||\hat{J}_k||).
\end{align} 
Since $\{Q_{n,k}(s,a,y)\}$, $\{\xi(s)\}$ and $\{c_n(s,a)\}$ are bounded for all agents $n$ and for all $s \in \mathcal{S}, a \in \mathcal{A}$ and $y \in \mathcal{Y}$, we have that $\{||\hat{U}_k||\}$ and $\{||\hat{J}_k||\}$ are bounded. Using Lemma 4.1 of \cite{kar2013cal} we can conclude that $P((k+1)^\tau \hat{z}_k \rightarrow 0)=1$ as $k \rightarrow \infty$ for all $\tau \in (0,\tau_1-\tau_2-\frac{1}{2})$.  Therefore, we have $P(\hat{z}_k \rightarrow 0)=1 \Rightarrow P(z_k \rightarrow \frac{1}{N} 1_N^T z_k 1_N)=1 \Rightarrow P(Q_{n,k}(s,a,y) \rightarrow \bar{Q}_n(s,a,y))=1$ for all $n$ and all $s \in \mathcal{S}, a \in \mathcal{A}$ and $y \in \mathcal{Y}$ as $k \rightarrow \infty$. Hence we recover Eqn. \eqref{qdl_eq1}.
\end{proof}

In order to solve the maximization over $\xi$ in Eqn. \eqref{qd_l_update} effectively, we now establish that the CVaR QD-learning update preserves the concavity of $\{y~Q_{n,k}(s,a,y) \}$. 
We observe that this concern is unique to the CVaR-based update, and is not seen in the expectation-based QD-learning update in \cite{kar2013cal}. 
Our next result formalizes this insight. 


\begin{theorem} \label{thrm: concav}
Suppose $\{y~Q_{n,k}(s,a,y) \}$ is concave in $y$ for all agents  $n$ and for all $s \in \mathcal{S}, a \in \mathcal{A}$ and $y \in \mathcal{Y}$. Then, $\{y~Q_{n+1,k}(s,a,y)\}$ is also concave in $y$ for all agents $n$ and for all $s \in \mathcal{S}, a \in \mathcal{A}$ and $y \in \mathcal{Y}$.
\end{theorem}

\begin{proof}
Let, $y_1,y_2 \in \mathcal{Y}$, $\lambda \in [0,1]$ and $y_\lambda=(1-\lambda)y_1+ \lambda y_2$. Then, we can write
\begin{align*}
&(1-\lambda)y_1 Q_{n,k+1}(s,a,y_1)+ \lambda y_2 Q_{n,k+1}(s,a,y_2) \\
&= (1-\alpha_k-|\mathcal{N}(n)|\beta_k)\Big((1-\lambda)y_1 Q_{n,k}(s,a,y_1)\\
&+\lambda y_2 Q_{n,k}(s,a,y_2)\Big) +\beta_k \sum_{l \in \mathcal{N}(n)} \big((1-\lambda )y_1 Q_{l,k}(s,a,y_1) \\
&+\lambda y_2 Q_{l,k}(s,a,y_2)\big) +\alpha_k \Big( ((1-\lambda)y_1+ \lambda y_2) c_n(s,a)+ \\
& \gamma \Big( \min_{a_1' \in \mathcal{A}} \big[ \max_{\xi_1 \in \mathcal{U}_{CVaR}(y,P(\cdot|s,a))} \xi_1(s') (1-\lambda)y_1 \\
&Q_{n,k}(s',a',y_1  \xi_2(s)) \big] + \min_{a_2' \in \mathcal{A}} \big[ \max_{\xi_2 \in \mathcal{U}_{CVaR}(y,P(\cdot|s,a))} \xi_2(s') y_2\\
&Q_{n,k}(s',a_2',y_2  \xi_2(s)) \big] \Big)\\
& \leq (1-\alpha_k-|\mathcal{N}(n)|\beta_k) y_\lambda Q_{n,k}(s,a,y_\lambda)+ \beta_k \sum_{l \in \mathcal{N}(n)} (y_\lambda \\
& Q_{l,k}(s,a,y_\lambda))+ \alpha_k \big( y_\lambda c_n(s,a)+ \gamma \big( \min_{a' \in \mathcal{A}} \big[ \\
&\max_{\substack{\xi_1 \in \mathcal{U}_{CVaR}(y_1,P(\cdot|s,a)) \\ \xi_2 \in \mathcal{U}_{CVaR}(y_2,P(\cdot|s,a))}} \xi_1(s') (1-\lambda)y_1 Q_{n,k}(s',a',y_1  \xi_2(s)) \\
&+\xi_2(s') y_2 Q_{n,k}(s',a_2',y_2  \xi_2(s)) \big] \big) \\
& \leq (1-\alpha_k-|\mathcal{N}(n)|\beta_k) y_\lambda Q_{n,k}(s,a,y_\lambda)+ \beta_k \sum_{l \in \mathcal{N}(n)} (y_\lambda \\
& Q_{l,k}(s,a,y_\lambda))+ \alpha_k \big( y_\lambda c_n(s,a)+ \gamma \big( \min_{a' \in \mathcal{A}} \\
&\big[ \max_{\substack{\xi_1 \in \mathcal{U}_{CVaR}(y_1,P(\cdot|s,a)) \\ \xi_2 \in \mathcal{U}_{CVaR}(y_2,P(\cdot|s,a))}} ((1-\lambda)y_1 \xi_1(s')+ \lambda y_2 \xi_2(s')) \\
& Q_{n,k}(s',a',(1-\lambda)y_1 \xi_1(s') + \lambda y_2 \xi_2(s'))\big] \big)
\end{align*}
Note that $(1-\alpha_k-|\mathcal{N}(n)|\beta_k) \geq 0$ and $\beta_k \geq 0$. In the above, we have used the concavity of $min$ operation and the concavity of positive linear combination of concave functions and our assumption that $\{y~Q_{n,k}(s,a,y) \}$ is concave in $y$ for all agent $n$ and for all $s \in \mathcal{S}, a \in \mathcal{A}$ and $y \in \mathcal{Y}$.

We define $\xi:= \frac{(1-\lambda)y_1 \xi_1(s')+ \lambda y_2 \xi_2(s')} {y_\lambda}$. Note that when $\xi_1 \in \mathcal{U}_{CVaR}(y_1,P(\cdot|s,a))$ and $\xi_2 \in \mathcal{U}_{CVaR}(y_2,P(\cdot|s,a))$ we can write $\xi \in [0, \frac{1}{y_\lambda}]$, where $\sum_{s'\in \mathcal{S}} \xi(s') P(s'|s,a)=1$. Thus, we can write
\begin{align*}
& (1-\lambda)y_1 Q_{n,k+1}(s,a,y_1)+ \lambda y_2 Q_{n,k+1}(s,a,y_2) \\
& \leq (1-\alpha_k-|\mathcal{N}(n)|\beta_k) y_\lambda Q_{n,k}(s,a,y_\lambda)+ \beta_k \sum_{l \in \mathcal{N}(n)} (y_\lambda \\
& Q_{l,k}(s,a,y_\lambda))+ \alpha_k \big( y_\lambda c_n(s,a)+ \gamma \big( \min_{a' \in \mathcal{A}} \\
&\big[ \max_{\xi \in \mathcal{U}_{CVaR}(y,P(\cdot|s,a))} y_\lambda \xi(s')  Q_{n,k}(s',a',y_\lambda \xi(s'))\big] \big) \\
&= y_\lambda Q_{n,k+1}(s,a,y_\lambda), 
\end{align*}
which completes the proof.
\end{proof}

\section{The CVaR QD-Learning Algorithm} \label{CVaRAlgo}

Theorems \ref{CVaRQDConsensus} and \ref{thrm: concav} are the critical components in the design of our \emph{CVaR QD-Learning} algorithm. 
This section presents our algorithm, and describes how some challenges in the implementation of the algorithm are overcome. 

First, the parameter $y$ in the Q-value $Q_{n,k}(s,a,y)$ takes values in the contiguous interval $\mathcal{Y} = (0,1]$. 
We overcome this challenge in the manner proposed in \cite{chow2015risk} by discretizing $\mathcal{Y}$ into sufficiently large number of intervals $m$, and considering only the extremities of each interval. 
Then, we can rewrite $\mathcal{Y} = [y_1, \cdots,  y_m]$, where $0 < y_1 < \cdots < y_m \leq 1$. 

The second challenge arises from the maximization of $\xi(s_{k+1}) Q_{n,k}(s_{k+1},a',y_k  \xi(s_{k+1}))$ over $\xi$ in the update in Eqn. \eqref{qd_l_update}. 
Our algorithm overcomes this challenge by solving a modified maximization problem, 
\begin{align*}
    \frac{1}{y_k} \max_{\xi(s_{k+1})} \big[(y_k\xi(s_{k+1})) Q_{n,k}(s_{k+1}, a',y_k  \xi(s_{k+1})) \big]. 
\end{align*}
The concavity property proved in Theorem \ref{thrm: concav} then allows us to conclude that any local maximum points of $\big[(y_k\xi(s_{k+1})) Q_{n,k}(s_{k+1}, a',y_k  \xi(s_{k+1})) \big]$ is indeed a global maximum. 

The final challenge is identifying an admissible value for ${\xi(s_{k+1})}$ during the maximization step since the transition probabilities $P(s_{k+1}|s_k,a_k)$ in reinforcement learning are unknown and revealed to agents only during interactions with the environment. 
Our implementation addresses this challenge by initially choosing ${\xi(s')}=1$ for all $s' \in \mathcal{S}$. 
Then, at every iteration of the \emph{CVaR QD-Learning} algorithm, we update an upper bound ${\bar{\xi}(s'|s,a)}$ in a manner such that ${\bar{\xi}(s'|s,a)} \hat{P}(s'|s,a) \leq 1$, where $\hat{P}(s'|s,a)$ is an estimate of $P(s'|s,a)$. 
We use a standard assumption in` Q-learning that a (stochastic) base policy $\pi_0$ is chosen such that every state-action pair is visited infinitely often~\cite{sutton2018reinforcement, watkins1992q} to compute an estimate $\hat{P}(s'|s,a)$ and thus obtain the bound ${\bar{\xi}(s'|s,a)}$. 
We use ${\bar{\xi}(s'|s,a)}= \frac{1}{\hat{P}(s'|s,a)}$ if $\hat{P}(s'|s,a) \neq 0$ (otherwise we keep the initial guess ${\bar{\xi}(s'|s,a)}=1$). For the maximization over $\xi$, we will use any $\xi(s_{k+1}) \leq \bar{\xi}(s_{k+1}|s_k,a_k)$ such that $\xi(s_{k+1})~y_k \in \mathcal{Y}$. 

The steps of our \emph{CVaR QD-Learning} algorithm is detailed in Algorithm \ref{algo:opt_policy}.

\begin{center}
  	\begin{algorithm}[!htp]
  		\caption{CVaR $QD$-Learning}
  		\label{algo:opt_policy}
  		\begin{algorithmic}[1]
  			\State \textbf{Input}: $y$, $\{\alpha_k\}$, $\{\beta_k\}$, $\gamma$, $\mathcal{S}$, $\mathcal{A}$, $\mathcal{Y}=[y_1, \cdots ,y_m]$, $N$.
                \State \textbf{Initialization}: $\{Q_{n,0}(s,a,y)\}$, $\bar{\xi}(s'|s,a)=1$, $k=0$.
  			\State \textbf{Output}: $\{Q_n(s,a,y)\}$.
  		    \Loop~{for each episode}
                 \State Initialize $s_0 \in \mathcal{S}$
                \Loop~{take action $a_k$ in $s_k$ 
                and observe next state $s_{k+1}$}.
                \State Update $\bar{\xi}(s_{k+1}|s_k,a_k)$.
                \State Each agent $n$ locally update $\{Q_{n,k+1}(s_k,a_k,y)\}$ according to \eqref{qd_l_update} for all $y \in \mathcal{Y}$ using any $\xi(s_{k+1}) \leq \bar{\xi}(s_{k+1}|s_k,a_k)$ such that $\xi(s_{k+1})~y_k \in \mathcal{Y}$.
                \State $k+1 \leftarrow k$.
                \EndLoop~{until $s$ is a terminal state}.
                \EndLoop
    	\end{algorithmic}
  \end{algorithm}
\end{center}

Since all agents asymptotically reach a consensus (Theorem \ref{CVaRQDConsensus}), a greedy policy over any agent's evaluation of the augmented action-value functions can be used to select the desired policy. Hence, the desired policy for the confidence level $y$ of CVaR can be selected as $\mu(s,y)= lim_{k \rightarrow \infty} \argmin_a Q_{n,k}(s,a,y)$ for any $n= 1, \cdots, N$.

\section{Experimental Evaluation}\label{sec:Experiments}

In this section, we carry out experiments to evaluate the performance of the \emph{CVaR QD-Learning} algorithm (Algorithm \ref{algo:opt_policy}). 
We first describe the experiment environment, and then report our results. 

\subsection{Environment}

Our setup consists of a network of $N=40$ agents. Each agent communicates with two of its nearest neighbors~\cite{kar2013cal}.  
State and action spaces are each binary-valued, thus giving $|\mathcal{S} \times \mathcal{U}|=4$, and $8$ different transition probabilities. 
Four of these transition probabilities are selected randomly via a uniform sampling from the interval $[0,1]$; this fixes the remaining four transition probabilities. 
For each agent $n$, the cost $c_n(s,a)$ is chosen from a uniform distribution that has a different mean for each state-action pair $(s,a)$. 
We set the discount factor $\gamma=0.7$, and parameters in Eqns. \eqref{eq:alpha_orig} and \eqref{eq:beta_orig} are chosen to be $\tau_1=0.2$, $\tau_2=0.3$, $a=0.2$ and $b=0.1$. 
The interval $(0,1]$ is discretized into $100$ equally spaced intervals to quantify confidence levels associated with CVaR. 
Thus, we have $\mathcal{Y}=\{0.01, 0.02, \cdots,0.99, 1\}$. 
We evaluate Algorithm \ref{algo:opt_policy} by instantiating a single trajectory $\{s_t,a_t\}$. At each step $t$ in state $s_t$, actions $a_t$ is chosen randomly via uniform sampling; the next state $s_{t+1}$ is determined by the transition probabilities. 
The initial state $s_0$ is chosen randomly. 
We also set the initial estimates of $Q-$values for the agents to different values.

\subsection{Evaluation Results}

\begin{figure}
    \centering
    \includegraphics[scale = 0.16]{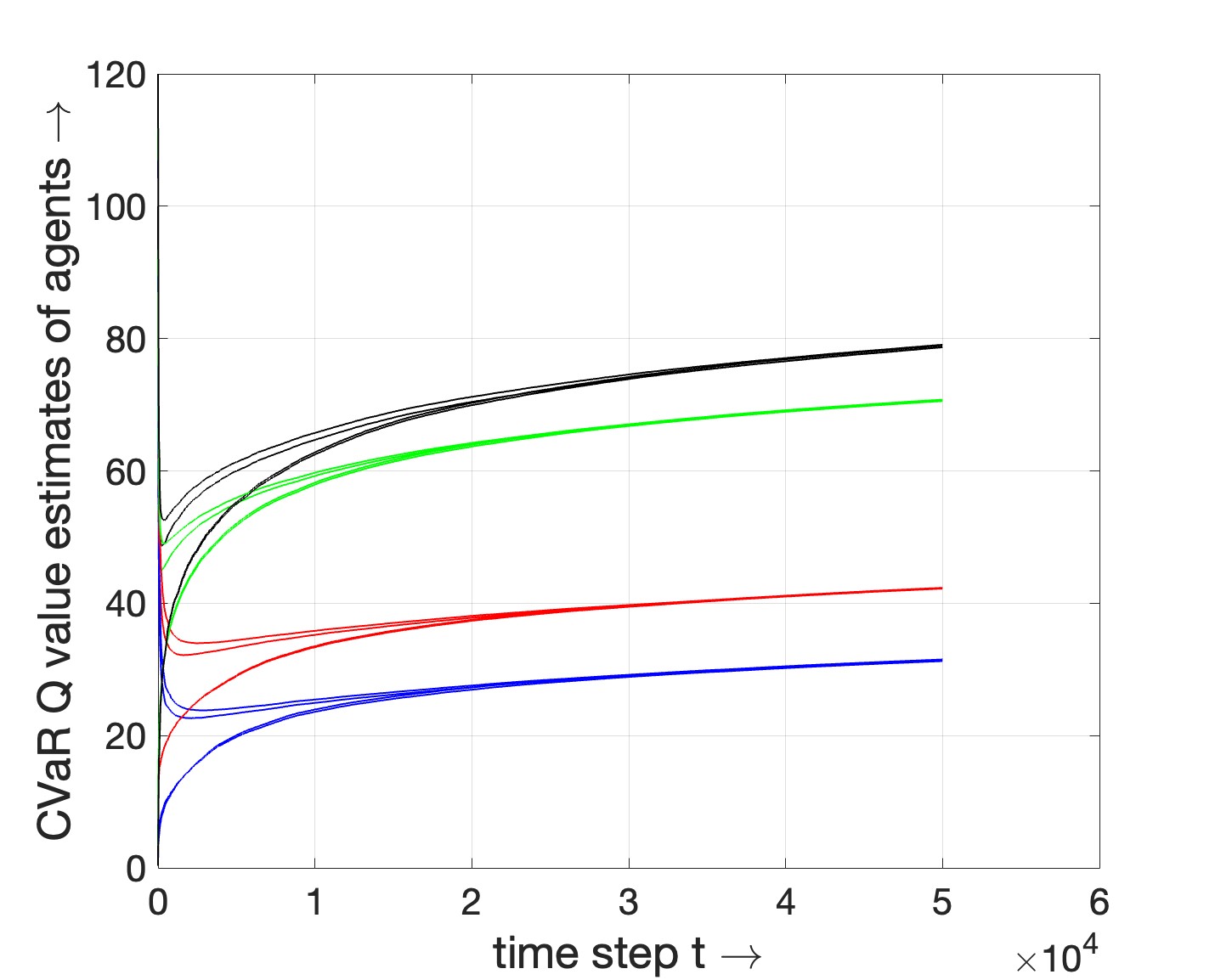}
    \caption{This figure presents the evolution of $CVaR$ $Q-$value estimates for four different non-neighbor agents with CVaR confidence level $y=0.7$. Each state-action pair $(s,a)$ is represented by a different color. We observe that when following Algorithm \ref{algo:opt_policy}, $CVaR$ $Q-$values of the agents reaches consensus for all $(s,a)$ pairs. 
    }
    \label{fig:fig1}
\end{figure}

We show the evolution of the $CVaR$ $Q-$value estimates when following our \emph{CVaR QD-Learning} algorithm (Algorithm \ref{algo:opt_policy}) for four different non-neighbor agents in Fig. \ref{fig:fig1}, where each state-action pair $(s,a)$ is represented by a different color. 
The CVaR confidence level is set to $y=0.7$ for all state-action pairs. 
We observe that agents asymptotically reach consensus on estimates of their $CVaR$ $Q-$values for all $(s,a)$ pairs. 
The rate of convergence is proportional to the number of times the state-action pair is visited in the trajectory (since a larger number of samples will be available in order to update the corresponding $CVaR$ $Q-$value).

\begin{figure}
    \centering
    \includegraphics[scale = 0.16]{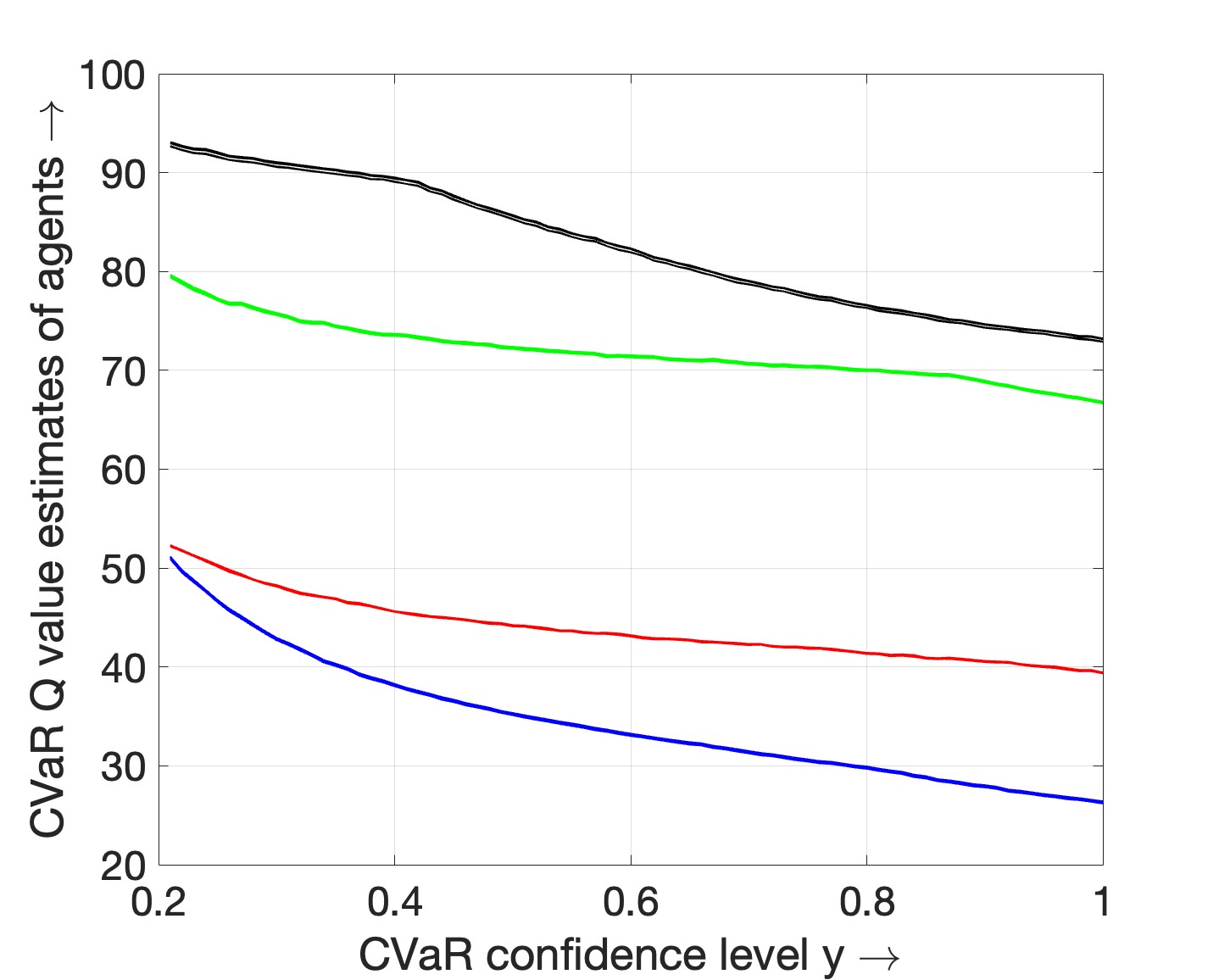}
    \caption{This figure presents the change in $CVaR$ $Q-$value estimates at time-step $t=50000$ for four non-neighbor agents when the CVaR confidence level $y \in [0.2,1]$. Each color denotes the $CVaR$ $Q-$value for one $(s,a)$ pair. We observe that agents' risk sensitivity varies from risk-aware ($y<1$) to risk-neutral ($y=1$), the $CVaR$ $Q-$value decreases. Intuitively, this shows that risk-aware behaviors will result in higher ($CVaR$) $Q-$values when the objective is to minimize an accumulated cost. 
    }
    \label{fig:fig2}
\end{figure}

Fig. \ref{fig:fig2} demonstrates the variation of $CVaR$ $Q-$value estimates at time-step $t=50000$ for different values of the confidence level $y \in [0.2,1]$. 
As agents reach consensus, we observe that for all state-action pairs, the $CVaR$ $Q-$value decreases with increase in $y$. 
The $CVaR$ $Q-$value is lowest when $y=1$, i.e., the situation identical to the more conventional expectation-based $Q-$value~\cite{sutton2018reinforcement, kar2013cal}. 
Intuitively, this result indicates that as agents' risk sensitivity varies from risk-neutral ($y=1$) to risk-aware ($y < 1$), their respective ($CVaR$) $Q-$values will be higher (since their objective is to minimize an accumulated cost). 

\section{Conclusion}\label{sec:Conclusion}

In this paper, we proposed a distributed multi-agent reinforcement learning (MARL) framework for decision-making by learning risk-aware policies. 
We used the conditional value-at-risk (CVaR) to characterize a risk-sensitive cost function, and introduced a Bellman operator to describe a CVaR-based state-action value function. 
Theoretically, we proved that this operator was a contraction, and that it converged to the optimal value. 
We used this insight to develop a distributed MARL algorithm called the \emph{CVaR QD-Learning} algorithm, and proved that risk-aware value functions associated to each agent reached consensus. 
We presented solutions to multiple challenges that arose during the implementation of the \emph{CVaR QD-Learning} algorithm, and evaluated its performance through experiments through simulations. We also demonstrated the effect of a risk parameter on value functions of agents when they reach consensus. 

One possible extension of our approach is to investigate the setting when some agents may be malicious or corrupt. 
Some preliminary work in this direction has been studied, albeit while minimizing an average cost criterion~\cite{xie2021towards}. 
Another interesting problem is to examine the continuous state-action setup, where policies are parameterized by (deep) neural networks. 
Initial research in the design of risk-sensitive policies for MARL have focused on the centralized training regime~\cite{qiu2021rmix}. 
Examining the development of resilient solutions in the context of synthesizing distributed risk-aware policies and developing distributed algorithms to characterize risk-aware behaviors in continuous state-action spaces is a promising research direction. 
%
\bibliographystyle{IEEEtran}
\bibliography{CDC23-RiskMARL-References}

\end{document}